\documentclass[twoside]{article}
\usepackage[utf8]{inputenc}
\usepackage[english]{babel}

\usepackage{amssymb}
\usepackage{amsmath}
\usepackage{amsthm}
\usepackage{verbatim}
\usepackage{color}
\usepackage{hyperref}
\usepackage[capitalise]{cleveref}
\usepackage[margin = 1.4in]{geometry}

\usepackage{algorithm2e}
\usepackage{appendix}
\usepackage{xcolor}
\usepackage{thmtools}
\usepackage{thm-restate}
\usepackage{hyperref}
\usepackage{cleveref}

\usepackage{amsmath}
\usepackage{float}
\usepackage{enumitem}

\newcommand{\Var}{\mbox{Var}}

\newcommand{\bbR}{\mathbb{R}}
\newcommand{\bbE}{\mathbb{E}}
\newcommand{\bbP}{\mathbb{P}}
\newcommand{\calF}{\mathcal{F}}

\newcommand{\calN}{\mathcal{N}}
\newcommand{\poly}{\mathrm{poly}}

\newcommand{\xtil}{\tilde{x}}
\newcommand{\ytil}{\tilde{y}}

\newcommand{\Perr}{\mathbb{P}({\rm error})}

\usepackage{hyperref}

\newtheorem{theorem}{Theorem}  
\newtheorem{lemma}[theorem]{Lemma}  
\newtheorem{definition}[theorem]{Definition}
\usepackage{amssymb}

\usepackage[bottom]{footmisc}

\usepackage[accepted]{aistats2020}
\usepackage{caption}
\usepackage{cite}
\usepackage{subcaption}

\usepackage[square,sort,comma]{natbib}

\setcitestyle{authoryear,open={[},close={]}}
\allowdisplaybreaks

\begin{document}

%

%

	\setlength{\parindent}{0pt}
	\setlength{\parskip}{10pt}

\twocolumn[

\aistatstitle{Learning Gaussian Graphical Models via Multiplicative Weights}

\aistatsauthor{Anamay Chaturvedi \And Jonathan Scarlett}
\aistatsaddress{Khoury College of Computer Sciences \\ Northeastern University \\ \texttt{chaturvedi.a@northeastern.edu} \And Depts.~Computer Science \& Mathematics \\ National University of Singapore \\ \texttt{scarlett@comp.nus.edu.sg}} ]

\begin{abstract}
    Graphical model selection in Markov random fields is a fundamental problem in statistics and machine learning.  Two particularly prominent models, the Ising model and Gaussian model, have largely developed in parallel using different (though often related) techniques, and several practical algorithms with rigorous sample complexity bounds have been established for each.  In this paper, we adapt a recently proposed algorithm of Klivans and Meka (FOCS, 2017), based on the method of multiplicative weight updates, from the Ising model to the Gaussian model, via non-trivial modifications to both the algorithm and its analysis.  The algorithm enjoys a sample complexity bound that is qualitatively similar to others in the literature, has a low runtime $O(mp^2)$ in the case of $m$ samples and $p$ nodes, and can trivially be implemented in an online manner.
\end{abstract}

	\section{Introduction}
	
	Graphical models are a widely-used tool for providing compact representations of the conditional independence relations between random variables, and arise in areas such as image processing \citep{Gem84}, statistical physics \citep{Gla63}, computational biology \citep{Dur98}, natural language processing \citep{Man99}, and social network analysis \citep{Was94}.  The problem of \emph{graphical model selection} consists of recovering the graph structure given a number of independent samples from the underlying distribution.  Two particularly prominent models considered for this problem are (generalized) Ising models and Gaussian models, and our focus is on the latter.
	
	In the Gaussian setting, the support of the sparse inverse covariance matrix directly corresponds to the graph under which the Markov property holds \citep{Wai08}: Each node in the graph corresponds to a variable, and any two variables are independent conditioned on a separating subset.

	
	
	In this paper, we present an algorithm for Gaussian graphical model selection that builds on the {\em multiplicative weights} approach recently proposed for (discrete-valued) Ising models \citep{klivans2017learning}.  This extension comes with new challenges due to the continuous and unbounded nature of the problem, prohibiting the use of several parts of the analysis in \citep{klivans2017learning} (as discussed more throughout the paper).  Under suitable assumptions on the (inverse) covariance matrix, we provide formal recovery guarantees of a similar form to other algorithms in the literature; see Section \ref{sec:contributions} and Theorem \ref{thm:main}.
	
	\subsection{Related Work} \label{sec:related}
	
	{\bf Learning Gaussian graphical models.} The problem of learning Gaussian graphical models (and the related problem of inverse covariance matrix estimation) has been studied using a variety of techniques and assumptions; our overview is necessarily brief, with a focus on those most relevant to the present paper.
	
	Information-theoretic considerations lead to the following algorithm-independent lower bound on the number of samples $m$ \citep{Wan10}:
	\begin{equation}
	    m = \Omega\bigg( \max\bigg\{ \frac{\log p}{\kappa^2},  \frac{d \log p}{\log(1+\kappa d)} \bigg\}\bigg), \label{eq:lower_bound}
	\end{equation}
    where $p$ is the number of nodes, $d$ the maximal degree of the graph, and $\kappa$ the minimum normalized edge strength (see \eqref{eq:kappa} below).  Ideally, algorithmic upper bounds would also depend only on $(p,d,\kappa)$  with no further assumptions, though as we describe below, this is rarely the case (including in our own results).
	
	Early algorithms such as SGS and PC \citep{Spi00,Kal07,Van13} adopted conditional independence testing methods, and made assumptions such as strong faithfulness.   A popular line of works studied the Graphical Lasso and related $\ell_1$-based methods \citep{Zho11,Hsi13,meinshausen2006high,yuan2007model,d2008first,Rav11}, typically attaining low sample complexities (e.g., $(d^2 + \kappa^{-2}) \log p$ \citep{Rav11}), but only under somewhat strong coherence-based assumptions.  More recently, sample complexity bounds were given under walk-summability assumptions \citep{Ana12a,kelner2019learning} and eigenvalue (e.g., condition number) assumptions \citep{Cai11,Wan16c,Cai16}.  Another line of works has adopted a Bayesian approach to learning Gaussian graphical models \citep{mohammadi2015bayesian,leppa2017learning}, but to our knowledge, these have not come with sample complexity bounds.
	
	Misra {\em et al.}~\citep{misra2017information}~provide an algorithm that succeeds with $m = O\big( \frac{d \log p}{\kappa^2} \big)$ without further assumptions, thus coming fairly close to the lower bound \eqref{eq:lower_bound}.  However, this is yet to be done efficiently, as the time complexity of $p^{O(d)}$ in \citep{misra2017information} (see also \citep[Thm.~11]{kelner2019learning}) is prohibitively large unless $d$ is small.  Very recently, efficient algorithms were proposed for handling general graphs under the additional assumption of attractivity (i.e., only having non-positive off-diagonal terms in the inverse covariance matrix) \citep{kelner2019learning}.


    {\bf Learning (generalized) Ising models.} Since our focus is on Gaussian models, we only briefly describe the related literature on Ising models, other than a particular algorithm that we directly build upon.
    
    Early works on Ising models relied on assumptions that prohibit long-range correlations \citep{Bre08,Ben09,Rav10,Jal11,Ana12a}, and this hurdle was overcome in a series of works pioneered by Bresler {\em et al.}~\citep{Bre14,Bre14a,Ham17}.  Recent developments have brought the sample complexity upper bounds increasingly close to the information-theoretic lower bounds \citep{San12}, using techniques such as interaction screening \citep{Vuf16}, multiplicative weights \citep{klivans2017learning}, and sparse logistic regression \citep{Wu19}.

	The present paper is particularly motivated by \citep{klivans2017learning}, in which an algorithm was developed for learning (generalized) Ising models based on the method of multiplicative weights.  More specifically, the algorithm constructs the underlying graph with a nearly optimal sample complexity and a low time complexity by using a weighted majority voting scheme to learn neighborhoods variable-by-variable, and updating the weights using Freund and Schapire's classic Hedge algorithm \citep{freund1997decision}. The proof of correctness uses the regret bound for the Hedge algorithm, as well as showing that approximating the distribution well according to a certain prediction metric ensures accurately learning the associated weight vector (and hence the neighborhood).
	
	\subsection{Contributions} \label{sec:contributions}

	In this paper, we adapt the approach of \citep{klivans2017learning} to Gaussian graphical models, and show that the resulting algorithm efficiently learns the graph structure with rigorous bounds on the number of samples required, and a low runtime of $O(mp^2)$ when there are $m$ samples and $p$ nodes.  As we highlight throughout the paper, each step of our analysis requires non-trivial modifications compared to \citep{klivans2017learning} to account for the continuous and unbounded nature of the Gaussian distribution.
	
	While we do not claim that our sample complexity bound improves on the state-of-the-art, it exhibits similar assumptions and dependencies to existing works that adopt condition-number-type assumptions (e.g., ACLIME \citep{Cai16}; see the discussion following Theorem \ref{thm:main}).   In addition, as highlighted in \citep{klivans2017learning}, the multiplicative weights approach enjoys the property of directly applying in the online setting (i.e., samples arrive one-by-one and must be processed, but not stored, before the next sample).  

    In Appendix \ref{runtime}, we discuss the runtimes of a variety of the algorithms mentioned in Section \ref{sec:related}, highlighting the fact that our $O(m p^2)$ runtime is very attractive.
	
	
	\section{Problem Statement} \label{sec:problem}

	Given a Gaussian random vector $X \in \mathbb{R}^p$ taking values in $\mathbb{R}^{p}$ with zero mean,\footnote{Our techniques can also handle the non-zero mean setting, but we find the zero-mean case to more concisely convey all of the relevant concepts.} covariance matrix $\Sigma \in \mathbb{R}^{p\times p}$, and inverse covariance matrix $\Theta \in \mathbb{R}^{p\times p}; \Theta = [\theta_{ij}]_{i,j \in [p]}$ (where $[p] = \{1,\dotsc,p\}$), we are interested in recovering the graph $G = (V,E)$ (with $V = [p]$) whose adjacency matrix coincides with the support of $\Theta$.  That is, we are interested in learning which entries of $\Theta$ are non-zero. 
	
	
	The graph learning is done using $m$ independent samples $(X^1,\dotsc,X^m)$ from $\calN(0,\Sigma)$.  Given these samples, the estimation algorithm forms an estimate $\hat{G}$ of the graph, or equivalently, an estimate $\hat{E}$ of the edge set, and the error probability is given by
    \begin{equation}
        \Perr = \bbP( \hat{G} \ne G ).
    \end{equation}
    We are interested in characterizing the worst-case error probability over all graphs within some class (described below).  Since our approach is based on neighborhood estimation, and each node has $p-1$ candidate neighbors, it will be convenient to let $n = p-1$.
	
	{\bf Definitions and assumptions.} Similarly to existing works such as \citep{Wan10,misra2017information}, our results depend on the minimum normalized edge strength, defined as
	\begin{align} 
        \kappa = \min_{(i,j) \in E} \left|  \frac{\theta_{ij}}{\sqrt{\theta_{ii}\theta_{jj}}}  \right|. \label{eq:kappa}
	\end{align}
    Intuitively, the sample complexity must depend on $\kappa$ because weaker edges require more samples to detect.

    We introduce some assumptions that are similar to those appearing in some existing works.  First, for each $i=1,\dotsc,p$, we introduce the quantity
	\begin{equation}
	    \lambda_i = \sum_{j \ne i} \bigg|\frac{\theta_{ij}}{\theta_{ii}}\bigg|, \label{eq:lambda}
	\end{equation}
	and we assume that $\max_{i \in [p]} \lambda_i$ is upper bounded by some known value $\lambda$.  As we discuss following our main result (Theorem \ref{thm:main}), this is closely related to an assumption made in \citep{Cai11,Cai16}, and as discussed in \citep{kelner2019learning}, the latter can be viewed as a type of condition number assumption, though eigenvalues do not explicitly appear.
	
	In addition, we define an upper bound $\theta_{\max}$ on the absolute values of the entries in $\Theta$, and an upper bound $\nu_{\max}$ on the variance of any marginal variable: 
	\begin{align} 
	    \theta_{\max} &= \max_{i,j} |\theta_{i,j}| = \max_{i} \theta_{i,i} \\
	    \nu_{\max} &= \max_{i} \Var[ X_i ]. \label{eq:nu}
	\end{align}
	A $(\theta_{\max}\nu_{\max})^2$ term appears in our final sample complexity bound (Theorem \ref{thm:main}).  This can again be viewed as a type of condition number assumption, since matrices with a high condition number may have large $\theta_{\max}\nu_{\max}$; e.g., see the example of \citep{misra2017information}.

    We will sometimes refer to the maximal degree $d$ of the graph in our discussions, but our analysis and final result will not depend on $d$.  Rather, one can think of $\lambda$ as implicitly capturing the dependence on $d$.
	
	For the purpose of simplifying our final expression for the sample complexity, we make some mild assumptions on the scaling laws of the above parameters:
	\begin{itemize}
	    \item We assume that $\lambda = \Omega(1)$.  This is mild since one can verify that $\lambda = \Omega(\kappa d)$, and the typical regimes considered in existing works are $\kappa d = \Theta(1)$ and $\kappa d \to \infty$ (e.g., see \citep{Wan10}).
	    \item We assume that $\lambda$, $\kappa$, $\nu_{\max}$, and $\theta_{\max}$ are in between $\frac{1}{\poly(p)}$ and $\poly(p)$.  This is mild since these may be high-degree polynomials, e.g., $p^{10}$ or $p^{100}$.
	\end{itemize}
	
	\section{Overview of the Algorithm}

	To recover the graph structure, we are first interested in estimating the inverse covariance matrix $\Theta \in \mathbb{R}^{p\times p}$ of the multivariate Gaussian distribution.
	
	For a zero-mean Gaussian random vector $X$, we have the following well-known result for any index $i$ (see Lemma \ref{gaussian} below):
	\begin{align}
	    \bbE[X_i|X_{\bar{i}}] =  \sum_{j \ne i} \frac{-\theta_{ij}}{\theta_{ii}}X_j = w^i \cdot X_{\bar{i}}, \label{eq:linear}
	\end{align}
	where $w^i = \big(\frac{-\theta_{ij}}{\theta_{ii}}\big)_{j\ne i}$, $X_{\bar{i}} = (X_j)_{j \ne i}$, and $a \cdot b$ denotes the dot product. In the related setting of learning Ising models and generalized linear models, the authors of \citep{klivans2017learning} used an analogous relation to turn the `unsupervised' problem of learning the inverse covariance matrix to a `supervised' problem of learning weight vectors given samples $(x^t, y^t)$, where the $x^t$ are $n$-dimensional tuples consisting of the values of $X_{\bar{i}}$, and $y^t$ are the values of $X_i$.  In particular, under the standard Ising model, the relationship analogous to \eqref{eq:linear} follows a logistic (rather than linear) relation.
	
	In \citep{klivans2017learning}, the Hedge algorithm of \citep{freund1997decision} is adapted to the problem of estimating the coefficients of $w^i$.  This is achieved for sparse generalized linear models (with bounded Lipschitz transfer functions) by first finding a vector $v$ that approximately minimizes an expected risk quantity with high probability, which we define analogously for our setting.  Their algorithm, referred to as Sparsitron, is shown in Algorithm \ref{alg:sparsitron}.  Note that both $(\xtil^t,\ytil^t)$ will represent suitably-normalized samples $(x^t,y^t)$ to be described below, and $(a^t,b^t)$ will represent further samples with the same distribution as $(\xtil^t,\ytil^t)$.
	
	\begin{definition}
		The \emph{expected risk} of a candidate $v \in \mathbb{R}^n$ for the neighborhood weight vector $w^i$ of a marginal variable $X_i$ is
		\begin{align} \varepsilon(v) := \bbE_X \left[\left(v\cdot X_{\bar{i}} -  w^i \cdot X_{\bar{i}} \right)^2 \right] \label{riskdefinition}.
		\end{align}
	\end{definition}

	\begin{algorithm}[!t]
        \hrule\medskip
		\KwData{$T+M$ normalized samples $\{(\xtil^t,\ytil^t)\}_{t=1}^T, \{(a^j,b^j)\}_{j=1}^M$; $\ell_1$-norm  parameter $\lambda$; update parameter $\beta$ (default value $\frac{1}{1+\sqrt{\frac{\ln n}{T}}}$)}
        \KwResult{Estimate of weight vector in $\bbR^n$}
		Initialize $v^0 = 1/n$\\
		\For{t=1,\dots, T}{
			$\cdot$ Let $p^t = \frac{v^{t-1}}{\|v^{t-1}\|_1}$\\
			$\cdot$ Define $l^t \in \mathbb{R}^n$ by $l^t = (1/2)(\textbf{1} + (\lambda p^t \cdot \xtil^t - \ytil^t) \xtil^t)$\\
			$\cdot$ Update the weight vectors: For each $i \in [n]$, set $v_i^t = v_i^{t-1} \cdot \beta^{l_i^t}$ 
		} 
		\For{t=1,\dots,M}{
			$\cdot$ Compute the empirical risk for each $t$:
			\begin{align} \hat{\varepsilon}(\lambda p^t) = \frac{\sum_{j=1}^M (\lambda p^t \cdot a^j - b^j)^2 }{M} \end{align}
		}
		\Return{$\lambda p^{t^*}$ for $t^* = \arg \min_{t \in [T]} \hat{\varepsilon}(\lambda p^t)$ \smallskip\hrule\medskip} 
		
		\caption{Sparsitron algorithm for estimating a weight vector $w \in \mathbb{R}^n$.  It is assumed here that the true weight vector has only positive weights and $\ell_1$-norm exactly $\lambda$ (see Footnote \ref{foot:omitted_step}). \label{alg:sparsitron}}
	\end{algorithm}
	
	\begin{algorithm}[!t]
        \hrule\medskip
		\KwData{$T+M$ samples, tuple $(\nu_{\max},\theta_{\max},\lambda,\kappa)$, target error probability $\delta$}
		\KwResult{Estimate of the graph}
		\For{$i=1,\cdots,p$}{
			$\cdot$ Normalize the $T$ samples as $(\xtil^t,\ytil^t) := \frac{1}{B\sqrt{\nu_{\max}(\lambda +1)}}(x^t,y^t)$ with $B = \sqrt{2 \log \frac{2pT}{\delta}}$, and similarly normalize the final $M$ samples to obtain $\{(a^j,b^j)\}_{j=1}^M$\\
			$\cdot$ Run Sparsitron on the normalized samples to obtain an estimate $v^i$ of the weight vector $w^i = \big(\frac{-\theta_{ij}}{\theta_{ii}}\big)_{j\ne i}$ of node $i$ 
		} 
        For every pair $i$ and $j$, identify an edge between them if $\max\{ |v^i_j|, |v^j_i|\}\ge 2\kappa/3$
        \smallskip\hrule\medskip 
		\caption{Overview of the algorithm for Gaussian graphical model selection. \label{alg:graph}}
	\end{algorithm}

	The Sparsitron algorithm uses what can be seen as a simple majority weighted voting scheme. For every possible member $X_j$ of the neighborhood of node $i$, the algorithm maintains a weight $v_j$ (which we think of as seeking to approximate the $j$-th entry of $w^i$), and updates the weight vector via multiplicative updates as in the Hedge algorithm. After $T$ such consecutive estimates, the algorithm uses an additional $M$ samples to estimate the expected risk for each of the $T$ candidates empirically, and then returns the candidate with the smallest empirical risk.
	
	As in \citep{klivans2017learning}, we assume without loss of generality that $w_i \geq 0$ for all $i$; for if not, we can map our samples $(x,y)$ to $([x,-x],y)$ and adjust the weight vector accordingly. We can also assume that $\lVert w \rVert_1$ equals its upper bound $\lambda$, since otherwise we can introduce a new coefficient and map our samples to $([x,-x,0],y)$.\footnote{\label{foot:omitted_step}This step is omitted in Algorithm \ref{alg:sparsitron}, so that we can lighten notation and work with vectors in $\bbR^n$ rather than $\bbR^{2n+1}$.  Formally, it can be inserted as an initial step, and then the resulting length $2n+1$ weight vector can be mapped back to a length-$n$ weight vector by taking the first $n$ entries and subtracting the second $n$ entries, while ignoring the final entry.  The initial part of our analysis considering Sparsitron can be viewed as corresponding to the case where the weights are already positive and the $\ell_1$-norm bound $\lambda$ already holds with equality, but it goes through essentially unchanged in the general case. \nopagebreak} If the true norm were $\lambda' < \lambda$, then the modified weight vector would have a value of $\lambda - \lambda'$ corresponding to the $0$ coefficient.

	Once the neighborhood weight vectors have been estimated, we recover the graph structure using thresholding, as outlined in Algorithm \ref{alg:graph}.  Here, $T$ and $M$ must satisfy certain upper bounds that we derive later (see Theorem \ref{thm:main}).  The overall sample complexity is $m = T + M$, and as we discuss following Theorem \ref{thm:main}, the runtime is $O(mp^2)$.  This runtime is compared to the runtimes of various existing algorithms for learning Gaussian graphical models in Appendix \ref{runtime}.

    \medskip
	\section{Analysis and Sample Complexity} \label{sec:proof}
	
	Our analysis proceeds in several steps, given in the following subsections.

	\subsection{Preliminary Results}
	
    \subsubsection{Properties of Multivariate Gaussians} \label{sec:propr_gauss}
	
	We first recall some results regarding multivariate Gaussian random variables that we will need throughout the analysis.
	
	\begin{restatable}{lemma}{gaussian}
		\label{gaussian}
		Given a zero-mean multivariate Gaussian $X = (X_1 ,\dots, X_{p})$ with inverse covariance matrix $\Theta = [\theta_{ij}]$, and given $T$ independent samples $(X^1,\dotsc,X^T)$ with the same distribution as $X$, we have the following:
		\begin{enumerate}
        \item For any $i \in [p]$, we have $X_i = \eta_i + \sum_{j \ne i} \big(-  \frac{\theta_{ij}}{\theta_{ii}}\big) X_j$, where $\eta_i$ is a Gaussian random variable with variance $\frac{1}{\theta_{ii}}$, independent of all $X_j$ for $j\ne i$.
        \item $\bbE[X_i|X_{\bar{i}}] = \sum_{j \ne i} \big(\frac{-\theta_{ij}}{\theta_{ii}}\big)X_j = w^i \cdot X_{\bar{i}}$, where $w^i = \big(\frac{-\theta_{ij}}{\theta_{ii}}\big)_{j\ne i} \in \mathbb{R}^n$ (with $n = p-1$).
        \item Let $\lambda$ and $\nu_{\max}$ be defined as in \eqref{eq:lambda} and \eqref{eq:nu}, set $B := \sqrt{2 \log \frac{2pT}{\delta}}$, and define $(\xtil^t,\ytil^t) := \frac{1}{B\sqrt{\nu_{\max}(\lambda +1})}(x^t,y^t)$, where $(x^t,y^t) = (X^t_{\bar{i}},X^t_i)$ for an arbitrary fixed coordinate $i$. Then, with probability at least $1-\delta$, $\ytil^t$ and all entries of $\xtil^t$ ($t = 1, \dots, T$) have absolute value at most $\frac{1}{\sqrt{\lambda+1}}$. 
        \end{enumerate}
	\end{restatable}
	\begin{proof}
		These properties are all standard and/or use standard arguments; see Appendix \ref{app:multivariate} for details.
	\end{proof}
	
	
	
	\subsubsection{Loss Guarantee for Sparsitron} \label{sec:sparistron}
	
	Recall that $n = p-1$.  In the proof of \citep[Theorem~3.1]{klivans2017learning}, it is observed that the Hedge regret guarantee implies the following.
	
	\begin{lemma}\label{sparsitronguarantee} {\em (\citep{klivans2017learning})} For any sequence of loss vectors $l^t \in [0,1]^n$ for $t = 1, \dots T$, the Sparsitron algorithm guarantees that
    \begin{equation}
        \sum_{t=1}^T p^t\cdot l^t \leq \min_{i\in [n]} \sum_{t=1}^T l_i^t + O(\sqrt{T\log n} + \log n) .
    \end{equation}
	\end{lemma}
	
	To run the Sparsitron algorithm, we need to define an appropriate sequence of loss vectors in $[0,1]^n$. Let 
	\begin{equation}
	    l^t = (1/2)(\textbf{1} + (\lambda p^t \cdot \xtil^t - \ytil^t) \xtil^t), \label{eq:losses}
	\end{equation}
	where $\textbf{1}$ is the vector of ones, and $\lambda p^t$ is Sparsitron's estimate at the beginning of the $t$-th iteration, formed using samples $1, \dots , t-1$. To account for the fact that the Hedge algorithm requires bounded losses for its regret guarantee, we use the high probability scaling in the third part of \cref{gaussian}: Since $p^t \in [0,1]^n$ and $\sum_{t=1}^n p_t = 1$, we have that $ |\lambda p^t \cdot \xtil^t - \ytil^t|<\sqrt{\lambda + 1}$, and that consequently $(\lambda p^t \cdot \xtil^t - \ytil^t) \xtil^t \in [-1,1]^n $.  It then follows that $l^t$, as defined in \eqref{eq:losses}, lies in $[0,1]^n$.
	Hence, Lemma \ref{sparsitronguarantee} applies with probability at least $1-\delta$ when we use $(\xtil^t,\ytil^t) := \frac{1}{B\sqrt{\nu_{\max}(\lambda +1)}}(x^t,y^t)$.
	
	\subsubsection{Concentration Bound for Martingales} \label{sec:conc_bound}
	
	Unlike the analysis of Ising (and related) models in \citep{klivans2017learning}, here we do not have the liberty of assuming bounded losses.  In the previous subsection, we circumvented this issue by noting that the losses are bounded with high probability, and such an approach is sufficient for that step due to the fact that the Hedge regret guarantee applies for arbitrary (possibly adversarially chosen) bounded losses.  However, while such a ``truncation'' approach was sufficient above, it will be insufficient (or at least challenging to make use of) in later parts of the analysis that rely on the Gaussianity of the samples.
	
	In this subsection, we present a concentration bound that helps to overcome this difficulty, and serves as a replacement for the Azuma-Hoeffding martingale concentration bound used in \citep{klivans2017learning}.
	
	Specifically, we use \citep[Lemma 2.2]{van1995exponential}, which states that given a martingale $M_t$, if we can establish Bernstein-like inequalities on the `sums of drifts' of certain higher order processes, 
	then we can establish a concentration bound on the main process. Here we state a simplified version for discrete-time martingales that suffices for our purposes (in \citep{van1995exponential}, continuous-time martingales are also permitted). This reduction from \citep[Lemma 2.2]{van1995exponential} is outlined in \cref{simplevdgproof}. 
	
	\begin{lemma}\label{simplevdg} {\em(\citep{van1995exponential})} 
		Let $M_t$ be a discrete-time martingale with respect to a filtration $\calF_t$ such that $\bbE[M_t^2] < \infty$ for all $t$, and define $\Delta M_t = M_t - M_{t-1}$ and $V_{m,t} = \sum_{j=1}^t \bbE[ |\Delta M_j|^m \,|\,\calF_{j-1}]$.
        Suppose that for all $t$ and some $0<K<\infty$, we have
		\begin{equation}
		    V_{m,t} \leq \frac{m!}{2} K^{m-2} R_t, \qquad m = 2, 3, \dots \label{eq:Vmt}
		\end{equation}
		for some process $R_t$ that is measurable with respect to $\calF_{t-1}$. Then, for any $a, b>0$, we have
        \begin{multline}
            \bbP(M_t \geq a \mbox{ and } R_t \leq b^2 \mbox{ for some }t) \\ \leq \exp\left( - \frac{a^2}{2aK + b^2} \right).
        \end{multline}
	\end{lemma}
	
	\subsection{Bounding the Expected Risk} \label{sec:bound_risk}
	
    For compactness, we subsequently write $w$ as a shorthand for the weight vector $w^i \in \mathbb{R}^n$ of the node $i$ whose neighborhood is being estimated.
	We recall the choice of $l^t$ in \eqref{eq:losses}, and make use of the following definitions from \citep{klivans2017learning}:
	\begin{align}
	Q^t &:= (p^t - w/\lambda)\cdot l^t \label{Qdefinition}\\ 
	Z^t &:= Q^t - \bbE_{t-1}[Q^t], \label{Zdefinition}
	\end{align}
    where here and subsequently, we use the notation $\bbE_t[\cdot] := \bbE[\cdot |(x^1,y^1),\dots,(x^t,y^t)]$ to denote conditioning on the samples up to time $t$.  The analysis proceeds by showing that $\sum_{j=1}^T Z^j$ is concentrated around zero, upper bounding the expected risk in terms of $\bbE_{t-1}[Q^t]$, and applying Sparsitron's guarantee from \cref{sparsitronguarantee}.

	
	We first use \cref{simplevdg} to obtain the following result.
	
	\begin{restatable}{lemma}{zconcentration}
		\label{zconcentration}
				$|\sum_{j=1}^T Z^j| = O\left(\sqrt{T \log \frac{1}{\delta}}\right)$ with probability at least $1-\delta$.
	\end{restatable}
	
	\begin{proof}
		The proof essentially just requires substitutions in \cref{simplevdg}. The martingale process is $M_t = \sum_{j\leq t} Z^j$, and we obtain $\Delta M_t = Z^t$, along with 
        \begin{equation}
            V_{m,t} = \sum_{j=1}^t \bbE_{j-1}[|Z^j|^m].
        \end{equation}
		The rest of the proof entails unpacking the definitions and using standard properties of Gaussian random variables to show that the Bernstein-like requirements are satisfied for the concentration bound in \cref{simplevdg}. The details are provided in Appendix \ref{app:concentration}.
	\end{proof}
	
	\begin{lemma} \label{lem:spars_risk}
		If Sparsitron is run with $T \ge \log n$, then
		\begin{equation}
            \min_{t\in [T]} \varepsilon (\lambda p^t) = O\left(\frac{\lambda (\lambda + 1 ) \nu_{\max} \log \frac{nT}{\delta} \left(\sqrt{T\log \frac{n}{\delta}}\right)}{T}\right)  \label{eq:risk_bound}
		\end{equation}
        with probability at least $1- \delta$.
	\end{lemma}
	
	\begin{proof}
		From the definition of $Q^t$ in \eqref{Qdefinition}, we have that
		\begin{align}
		&\bbE_{t-1} [Q^t] \nonumber \\
        &= \bbE_{t-1}[(p^t - (1/\lambda)w)\cdot l^t] \\
		&= \bbE_{t-1}[(p^t - (1/\lambda)w)\cdot (1/2)(\textbf{1} + (\lambda p^t \cdot \xtil^t - \ytil^t) \xtil^t)] \label{defl}\\
		&= \frac{1}{2}\bbE_{t-1}\bigg[\left(\sum_i p_i^t - \sum_i \frac{w_i}{\lambda}\right)  \nonumber \\
            &\qquad + (p^t - (1/\lambda)w)\cdot \xtil^t (\lambda p^t \cdot \xtil^t - \ytil^t)\bigg]\\
		&= \frac{1}{2}\bbE_{t-1}\bigg[\left( 1 - \frac{\lambda}{\lambda} \right) \nonumber \\
            &\qquad +  (p^t \cdot \xtil^t - (1/\lambda)w \cdot \xtil^t) (\lambda p^t \cdot \xtil^t - \ytil^t)\bigg] \label{sums1}\\
		&= \frac{1}{2}\bbE_{t-1}\left[(p^t \cdot \xtil^t - (1/\lambda)w \cdot \xtil^t) (\lambda p^t \cdot \xtil^t - \ytil^t)\right] \label{isafunction} \\
		&= \frac{1}{2\lambda}\bbE_{t-1}[(\lambda p^t \cdot \xtil^t - w \cdot \xtil^t)^2]  \label{expfirst}\\
		&\geq \frac{1}{2\lambda(\lambda +1)\nu_{\max} B^2} \varepsilon (\lambda p^t), \label{apply_risk_def}
		\end{align}
		where \eqref{defl} uses the definition of $l^t$ in \eqref{eq:losses}, \eqref{sums1} uses $\|w\|_1 = \lambda$ (see Footnote \ref{foot:omitted_step}) and $\sum_{i} p_i^t = 1$, \eqref{expfirst} follows by noting that $p^t$ is a function of $\{\xtil^i,\ytil^i\}_{i=1}^{t-1})$ and computing the expectation over $\ytil^t$ first (using the second part of Lemma \ref{gaussian}), and \eqref{apply_risk_def} uses the definition of expected risk in \eqref{riskdefinition}, along with $\xtil^t= \frac{1}{B\sqrt{\nu_{\max}(\lambda +1)}} x^t$.
		
	    Summing both sides above over $t=1,\dotsc,T$, we have the following with probability $1 - O(\delta)$:\footnote{In the analysis, we apply multiple results that each hold with probability at least $1-\delta$.  More precisely, $\delta$ should be replaced by $\delta/L$ when applying a union bound over $L$ events, but since $L$ is finite, this only amounts to a change in the constant of the $O(\cdot)$ notation in \eqref{eq:risk_bound}.}
		\begin{align}
		&\frac{1}{2\lambda(\lambda +1)\nu_{\max} B^2} \sum_{t=1}^T \varepsilon (\lambda p^t) \nonumber  \\
		&\leq \sum_{t=1}^T \bbE_{t-1} [Q^t] \\
		&= \sum_{t=1}^T (Q^t - Z^t) \label{defz} \\
		&\leq \sum_{t=1}^T Q^t + O\left(\sqrt{T \log \frac{1}{\delta}}\right) \label{usingconc} \\
		&\leq \sum_{t=1}^T (p^t - w/\lambda)\cdot l^t + O\left(\sqrt{T \log \frac{1}{\delta}}\right) \label{defQt} \\
		&\leq \min_{i\in [n]} \sum_{t=1}^T l_i^t - \sum_{t=1}^T (w/\lambda)\cdot l^t + O(\sqrt{T\log n} + \log n) \nonumber \\
        &\hspace*{4cm} + O\left(\sqrt{T \log \frac{1}{\delta}}\right), \label{usingguarantee} 
		\end{align}
		where \eqref{defz} follows from the the definition of $Z^t$ in \eqref{Zdefinition}, \eqref{usingconc} uses Lemma \ref{zconcentration}, \eqref{defQt} follows from the definition of $Q^t$ in \eqref{Qdefinition}, and \eqref{usingguarantee} follows from \cref{sparsitronguarantee}.
		
		Since $\|w\|_1 = \lambda$ (see Footnote \ref{foot:omitted_step}), $ \min_{i\in [n]} \sum_{t=1}^T l_i^t - \sum_{t=1}^T (w/\lambda)\cdot l^t \leq 0$.  It follows from \eqref{usingguarantee} that
		\begin{multline}
		\frac{1}{2\lambda(\lambda +1)\nu_{\max} B^2} \sum_{j=1}^T \varepsilon (\lambda p^t) \\ = O\left(\sqrt{T\log n} + \log n + \sqrt{T \log \frac{1}{\delta}}\right),
		\end{multline}
		and substituting $B = \sqrt{2 \log \frac{2(n+1)T}{\delta}}$ gives 
		\begin{multline}
		\min_{t\in [T]} \varepsilon (\lambda p^t) = O\bigg(\frac{\lambda (\lambda + 1 ) \nu_{\max} \log \frac{nT}{\delta}}{T} \\ \times \Big(\sqrt{T\log n} + \log n+ \sqrt{T\log \frac{1}{\delta}}\Big)\bigg),
		\end{multline}	
        where we also lower bounded $\sum_{j=1}^T \varepsilon (\lambda p^t)$  by $T$ times the minimum value.  When $T \ge \log n$, the above bound simplifies to \eqref{eq:risk_bound}, as desired.
	\end{proof}
	
	Having ensured that that the minimal expected risk is small, we need the algorithm to identify a candidate whose expected risk is also sufficiently close to that minimum.  Sparsitron does this by using an additional $M$ samples to estimate the expected risk empirically.
	
	\begin{restatable}{lemma}{bernstein}
		\label{bernstein}
		For $\gamma>0$, $\rho \in (0,1]$, and fixed $v \in \mathbb{R}^n$ satisfying $\|v\|_1 \le \lambda$, there is some $M = O\big((\lambda+1) \frac{\log(1/\rho)}{\gamma}\big)$  such that
        \begin{align}    
        \bbP\left( \left| \frac{1}{M}\sum_{j=1}^M \Big( (v\cdot a^j  - b^j)^2 - \Xi\Big) - \varepsilon (v)\right| \geq \gamma\right)  \leq \rho,
        \end{align}
        where $\{(a^j,b^j)\}_{j=1}^M$ are the normalized samples defined in Algorithm \ref{alg:graph}, and $\Xi = \bbE\big[\mbox{\rm Var}[b^j \,|\,a^j]\big]$.\footnote{This quantity is the same for all values of $j$.}
	\end{restatable}
	
	\begin{proof}
        The high-level steps of the proof are to first establish the equality 
        \begin{align}
		      \bbE[(v\cdot a^j - b^j)^2] = \varepsilon (v) + \Xi,
		\end{align}
        and then use Bernstein's inequality to bound the deviation of $\sum_{j=1}^M \big( (v \cdot a^j - b^j)^2 - \Xi \big)$ from its mean value $\varepsilon(v)$.  The details are given in Appendix \ref{app:conc_risk}.
	\end{proof}
	
	\subsection{Graph Recovery and Sample Complexity} \label{sec:graph_rec}

    We complete the analysis of our algorithm in a sequence of three steps, given as follows.
	
	{\bf An $\ell_{\infty}$ bound.} We show that if our estimate $v$ approximates the true weight vector $w \in \mathbb{R}^n$ well in terms of the expected risk, then it also approximates it in the $\ell_{\infty}$ norm.  In \citep{klivans2017learning}, this was done using a property termed the `$\delta$-unbiased condition', whose definition relies on the underlying random variables being binary.  Hence, we require a different approach, given as follows.\footnote{See also \cite[Thm.~17]{kelner2019learning} for similar considerations under a different set of assumptions.}

	\begin{restatable}{lemma}{risktonorm} \label{risktonorm}
		Under the preceding setup, if we have $\varepsilon(v) \le \epsilon$, then we also have $\|v - w\|_\infty \le \sqrt{\epsilon \theta_{\mathrm{max}}}$, where $\theta_{\max}$ is a uniform upper bound on the diagonal entries of $\Theta$.
    \end{restatable}
	\begin{proof}
        The proof uses a direct calculation to establish that $\mathrm{Var}((v-w)\cdot X_{\bar{i}}) \ge |v_{i^*} - w_{i^*}|^2 \mathrm{Var}(\eta_{i^*})$ for a fixed index $i^*$; the details are given in Appendix \ref{pf_risktonorm}.
	\end{proof}
	
	Suppose that we would like to recover the true weight vector with a maximum deviation of $\epsilon'$ in any coordinate with probability at least $1-\delta$. By \cref{risktonorm}, we require $\epsilon$ to be no more than $(\epsilon')^2/\theta_{\max}$. We know from Lemma \ref{lem:spars_risk} that
	\begin{align} 
	\min_{t\in [T]} \varepsilon (\lambda p^t) 
	&= O\left(\frac{\lambda(\lambda+1) \nu_{\max} \log \frac{nT}{\delta} \left( \sqrt{T \log \frac{n}{\delta}}\right)}{T}\right),
	\end{align}
	from which we have that with $T = O\left(\frac{\lambda^2(\lambda+1)^2 \nu_{\max}^2}{\epsilon^2} \log^3 \frac{n}{\delta} \right)$,\footnote{The removal of $T$ in the logarithm $\log\frac{nT}{\delta}$ can be justified by the assumption that all parameters are polynomially bounded with respect to $p$ (see Section \ref{sec:problem}).} the minimum expected risk is less than $\epsilon/2$ with probability at least $1-\delta/2$. 
	
	From \cref{bernstein} with $\rho = \delta/(2T)$ and $\gamma = \epsilon/2$, we observe that we can choose $M$ satisfying 
	\begin{align}
	M &\leq O\bigg( (\lambda+1) \frac{\log(T/\delta)}{\epsilon} \bigg) \\
	&\leq O\bigg( (\lambda+1) \log \frac{\lambda(\lambda + 1) \nu_{\max} \log^{3/2} \frac{n}{\delta}}{\epsilon \delta} \bigg)
	\end{align}
	and estimate $\varepsilon(\lambda p^t) + \Xi$ (note that the second term doesn't affect the $\arg \min$) of the $T$ candidates $\lambda p^t$ within $\epsilon/4$ with probability at least $1-\frac{\delta}{2T}$. By the union bound (which blows the $\frac{\delta}{2T}$ up to $\frac{\delta}{2}$), the same follows for all $T$ candidates simultaneously. We then have that the candidate with the lowest estimate has expected risk within $\epsilon/2$ of the candidate with the lowest expected risk, and that the latter candidate's expected risk is less than $\epsilon/2$, so in sum the vector returned by the candidate has an expected risk less than $\epsilon$ with probability at least $1-\delta$. Moreover, the sample complexity is
	\begin{align}
	&T + M \nonumber \\
    &= O\left(\frac{\lambda^2(\lambda+1)^2 \nu_{\max}^2}{\epsilon^2} \log^3 \frac{n}{\delta} \right) \nonumber \\
        &~~~~ +  O\left((\lambda+1) \log \frac{\lambda(\lambda + 1) \nu_{\max} \log^{3/2} \frac{n}{\delta}}{\epsilon \delta}  \right) \\
	&= O\left(\frac{\lambda^4 \nu_{\max}^2}{\epsilon^2} \log^3 \frac{n}{\delta} \right), \label{eq:sample_init}
	\end{align}
	where the simplification comes by recalling from Section \ref{sec:problem} that $\lambda = \Omega(1)$ and all parameters are polynomially bounded with respect to $n$.  While the sample complexity \eqref{eq:sample_init} corresponds to probability at least $1-\delta$ for the algorithm of only a single $i \in [p]$, we can replace $\delta$ by $\delta/p$ and apply a union bound to conclude the same for all $i \in [p]$; since $p = n+1$, this only amounts to a chance in the constant of the $O(\cdot)$ notation.
	
	
	{\bf Recovering the graph.} Recall from \cref{risktonorm} that an expected risk of at most $\epsilon$ translates to a coordinate-wise deviation of at most $\epsilon' = \sqrt{\epsilon \theta_{\max}}$. We set $\epsilon = \frac{\kappa^2}{9\theta_{\max}}$, so that $\epsilon' = \frac{\kappa}{3}$.
	
	We observe that if $X_i$ and $X_j$ are neighbors, then \eqref{eq:kappa} yields the following lower bound:
    \begin{equation}
        \frac{\theta_{ij}^2}{\theta_{ii}\theta_{jj}} \geq \kappa^2
    \end{equation} 
    This ensures that at least one of the two values $|\theta_{ij}/\theta_{ii}|$ and $|\theta_{ij}/\theta_{jj}|$ must be greater than or equal to $\kappa$. On the other hand, if they are not neighbors, then the true value of both of these terms must be $0$. Since we have estimated all weights to within $\kappa/3$, it follows that any estimate of at least $2\kappa/3$ must arise from a true neighborhood relation (with high probability). Conversely, if there is a neighborhood relation, then at least one of the two factors $\theta_{ij}/\theta_{ii}$ and $\theta_{ij}/\theta_{jj}$ must have been found to be at least $2\kappa/3$.
	
	The method for recovering the graph structure is then as follows: For each possible edge, the weight estimates 
	$v^i_j$ and $v^j_i$ are calculated; if either of them is found to be greater than $2\kappa/3$, then the edge is declared to lie in the graph, and otherwise it is not.
	
	Substituting $\epsilon = \frac{\kappa^2}{9\theta_{\max}}$ into \eqref{eq:sample_init}, and recalling our notation $n = p-1$, we deduce the final sample complexity, stated as follows.
	
	\begin{theorem} \label{thm:main}
	    For learning graphs on $p$ nodes with minimum normalized edge strength $\kappa$, under the additional assumptions stated in Section \ref{sec:problem} with parameters $(\lambda,\nu_{\max},\theta_{\max})$, the algorithm described above attains $\Perr \le \delta$ with a sample complexity of at most
    	\begin{align}
    	m = O\left(\frac{\lambda^4 \nu_{\max}^2 \theta_{\max}^2}{\kappa^4} \log^3 \frac{p}{\delta} \right) \label{samplecomplexity}.
    	\end{align}
	\end{theorem}
	
	We can compare this guarantee with those of existing algorithms: As discussed in \citep[Remark 8]{kelner2019learning}, the $\ell_1$-based ACLIME algorithm \citep{Cai16} can be used for graph recovery with $m = O\big( \frac{\tilde{\lambda}^2 \log\frac{p}{\delta}}{\kappa^4} \big)$ samples, where $\tilde{\lambda}$ is an upper bound on the $\ell_1$ norm of any row of $\Theta$.  An algorithm termed HybridMB in \citep{kelner2019learning} achieves the same guarantee, and a greedy pruning method in the same paper attains a weaker $m = O\big( \frac{\tilde{\lambda}^4 \log\frac{p}{\delta}}{\kappa^6} \big)$ bound.
	
	The quantities $\lambda$ and $\tilde{\lambda}$ are closely related; for instance, in the case that $\theta_{ii} = 1$ for all $i$, we have $\tilde{\lambda} = 1+\lambda$.  More generally, if $\nu_{\max}$ and $\theta_{\max}$ behave as $\Theta(1)$, then our bound can be written as $O\big( \frac{\tilde{\lambda}^4}{\kappa^4} \log^3\frac{p}{\delta} \big)$, which is qualitatively similar to the bounds of \citep{Cai16,kelner2019learning} but with an extra $\big(\lambda \log\frac{p}{\delta}\big)^2$ term.
	
	We again highlight that our main goal is not to attain a state-of-the-art sample complexity, but rather to introduce a new algorithmic approach to Gaussian graphical model selection.  The advantages of this approach, as highlighted in \citep{klivans2017learning}, are low runtime and direct applicability to the online setting.
    In addition, as we discuss in the following section, we expect that there are parts of our analysis that could be refined to bring the sample complexity down further.

    {\bf Runtime.} The algorithm enjoys a low runtime similar to the case of Ising models \citep{klivans2017learning}: Sparsitron performs $m = T+M$ iterations that each require time $O(n) = O(p)$, for an overall runtime of $O(mp)$.  Since this is done separately for each $i=1,\dotsc,p$, the overall runtime is $O(mp^2)$.
	
	\section{Conclusion} \label{sec:conclusion}
	
	We have introduced a novel adaptation of the multiplicative weights approach to graphical model selection \citep{klivans2017learning} to the Gaussian setting, and established a resulting sample complexity bound under suitable assumptions on the covariance matrix and its inverse.  The algorithm enjoys a low runtime compared to existing methods, and can directly be applied in the online setting.
    
	The most immediate direction for further work is to seek refinements of our algorithm and analysis that can further reduce the sample complexity and/or weaken the assumptions made.  For instance, we normalized the samples to ensure a loss function in $[0,1]$ with high probability, and this is potentially more crude then necessary (and ultimately yields the $\log^3 p$ dependence).  One may therefore consider using an alternative to Hedge that is more suited to unbounded rewards.  In addition, various steps in our analysis introduced $\theta_{\max}$ and $\nu_{\max}$, and the individual estimation of diagonals of $\Sigma$ and/or $\Theta$ (e.g., as done in \citep{Cai16}) may help to avoid this.

	\section*{Acknowledgment} 

    This work was supported by the Singapore National Research Foundation (NRF) under grant number R-252-000-A74-281.

    \bibliography{refs}
	
    \newpage
    \appendix
    \onecolumn
    
    {\centering
     {\huge \bf Supplementary Material}
    
     {\Large \bf Learning Gaussian Graphical Models via Multiplicative Weights \\[2.5mm] \large (Anamay Chaturvedi and Jonathan Scarlett, AISTATS 2020) \par
    }
    
    
    }

    \bigskip
    
    All citations below are to the reference list in the main document.

    \section{Comparison of Runtimes} \label{runtime}

    Recall that $p$ denotes the number of nodes, $d$ denotes the maximal degree, $\kappa$ denotes the minimum normalized edge strength, and $m$ denotes the number of samples.
    The runtimes of some existing algorithms in the literature for Gaussian graphical model selection (see Section \ref{sec:related} for an overview) are outlined as follows:
    \begin{itemize}
        \item The only algorithms with assumption-free sample complexity bounds depending only on $(p,d,\kappa)$ have a high runtime of $p^{ O(d) }$, namely, $O(p^{2d+1})$ in \citep{misra2017information}, and $O(p^{d+1})$ in \citep[Thm.~11]{kelner2019learning}.
        \item  A greedy method in \citep[Thm.~7]{kelner2019learning} has runtime $O\big( \big(d\log\frac{1}{\kappa}\big)^3 m  p^2 \big)$. The sample complexity for this algorithm is $O\big( \frac{d}{\kappa^2} \cdot \log \frac{1}{\kappa} \cdot  \log n \big)$, but this result is restricted to attractive graphical models.
        \item  To our knowledge, $\ell_1$-based methods \citep{meinshausen2006high,yuan2007model,d2008first,Rav11,Cai11,Wan16c,Cai16} such as Graphical Lasso and CLIME do not have precise time complexities stated, perhaps because this depends strongly on the optimization algorithm used. We expect that a general-purpose solver would incur $O( p^3 )$ time, and we note that \citep[Table 2]{kelner2019learning} indeed suggests that these approaches are slower.
        \item  In practice, we expect BigQUIC \citep{Hsi13} to be one of the most competitive algorithms in terms of runtime, but no sample complexity bounds were given for this algorithm.   
        \item Under the local separation condition and a walk summability assumption, the algorithm of \citep{Ana12a} yields a runtime of $O( p^{2+\eta} )$, where $\eta > 0$ is an integer specifying the local separation condition.
    \end{itemize}

    Hence, we see that our runtime of $O(mp^2)$ is competitive among the existing works -- it is faster than other algorithms for which sample complexity bounds have been established.

	\section{Proof of Lemma \ref{gaussian} (Properties of Multivariate Gaussians)} \label{app:multivariate}
	
	We restate the lemma for ease of reference.
	
	\gaussian*
	
	\begin{proof}
	    The first claim is standard in the literature (e.g., see \citep[Eq.~(4)]{Zho11}), and the second claim follows directly from the first.

		For the third claim, let $N$ be a Gaussian random variable with mean $0$ and variance $1$. We make use of the standard (Chernoff) tail bound
        \begin{equation}
            \bbP\left(|N|>x\right) \le 2e^{-x^2/2}.
        \end{equation}
		By scaling the standard Gaussian distribution, recalling the definition of $\nu_{\max}$ in \eqref{eq:nu}, and using $B = \sqrt{2 \log \frac{2pT}{\delta}}$, it follows that
		\begin{align}
		\bbP(|x_i^t| > \sqrt{\nu_{\max}} B) &\leq \bbP\left(|N|> \sqrt{2\log \frac{2pT}{\delta}}\right) \\
		&\leq 2 \exp \left(-\log \frac{2pT}{\delta}\right) \\
		&\leq \frac{\delta}{pT},
		\end{align}
        and hence
        \begin{align}   
            \bbP\left(|x_i^t| > \frac{1}{\sqrt{\lambda + 1}} \right) &\leq \frac{\delta}{pT}.
        \end{align}
		The same high probability bound holds similarly for $\ytil^t$. By taking the union bound over these $p$ events, and also over $t=1,\dotsc,T$, we obtain the desired result.
	\end{proof}

    \section{Establishing \cref{simplevdg} (Martingale Concentration Bound)}\label{simplevdgproof}

    Here we provide additional details on attaining \cref{simplevdg} from a more general result in \citep{van1995exponential}.  While the latter concerns continuous-time martingales, we first state some standard definitions for discrete-time martingales.  Throughout the appendix, we distinguish between discrete time and continuous time by using notation such as $M_t,\calF_t$ for the former, and $\tilde{M}_t,\tilde{\calF}_t$ for the latter.
	
	\begin{definition} 
	    Given a discrete-time martingale $\{M_t\}_{t=0,1,\dots}$ with respect to a filtration $\{\calF_t\}_{t = 0,1,\dots}$, 
        we define the following:
	    \begin{enumerate}
	        \item The {\em compensator} of $\{M_t\}$ is defined to be
    		\begin{align}
    			V_t = \sum_{j=1}^t \bbE[M_j- M_{j-1} \,|\, \calF_{j-1}]. \label{eq:compensator}
    		\end{align}
    		\item A discrete-time process $\{W_t\}_{t=1,2,\dots}$ defined on the same probability space as $\{M_t\}$ is said to be {\em predictable} if $W_t$ is measurable with respect to $\calF_{t-1}$.
            \item We say that $\{M_t\}$ is {\em locally square integrable} if there exists a sequence of stopping times $\{\tau_k\}_{k=1}^{\infty}$ with $\tau_k \to \infty$ such that $\bbE[ M_{\tau_k}^2 ] < \infty$ for all $k$.
	    \end{enumerate}
	\end{definition}

    In the continuous-time setup of \cite[Lemma 2.2]{van1995exponential}, the preceding definitions are replaced by generalized notions, e.g., see \citep{Lip89}.  Note that the notion of a compensator in the continuous-time setting is much more technical, in contrast with the explicit formula \eqref{eq:compensator} for discrete time.

    The setup of \citep{van1995exponential} is as follows: Let $\{ \tilde{M}_t \}_{t\geq 0}$ be a locally square integrable continuous-time martingale with respect to to a filtration $\{\tilde{\calF}_t\}_{t\geq 0}$ satisfying right-continuity ($\tilde\calF_t = \cap_{s>t}\tilde\calF_s$) and completeness ($\calF_0$ includes all sets of null probability). For each $t > 0$, the martingale jump is defined as $\Delta \tilde{M}_t = \tilde{M}_t - \tilde{M}_{t-}$, where $t_-$ represents an infinitesimal time instant prior to $t$.  For each integer $m \ge 2$, a higher-order variation process $\{\sum_{s\leq t} |\Delta \tilde{M}_s|^m\}$ is considered, and its compensator is denoted by $\tilde{V}_{m,t}$.  Then, we have the following.

	
	\begin{lemma} \label{lem:cont_martingale}
        {\em \citep[Lemma~2.2]{van1995exponential}}
	    Under the preceding setup for continuous-time martingales, suppose that for all $t\geq 0$ and some $0<K<\infty$, it holds that
        \begin{equation}
            \tilde{V}_{m,t} \leq \frac{m!}{2} K^{m-2} \tilde{R}_t, \qquad m = 2, 3, \dots, 
        \end{equation}
		for some predictable process $\tilde{R}_t$. Then, for any $a, b>0$, we have
        \begin{equation}
            \bbP(\tilde{M}_t \geq a \mathrm{\:and\: } \tilde{R}_t \leq b^2 \mathrm{\:for\:some\:}t) \leq \exp\left( - \frac{a^2}{2aK + b^2} \right).
        \end{equation}
	\end{lemma}
	
	While Lemma \ref{lem:cont_martingale} is stated for continuous-time martingales, we obtain the discrete-time version in \cref{simplevdg} by considering the choice $\tilde{M}_t = M_{\lfloor t \rfloor}$, where $\{M_t\}_{t=0,1,\dotsc}$ is the discrete-time martingale.  Due to the floor operation, the required right-continuity condition on the continuous-time martingale holds.
    Moreover, the definition of a compensator in \eqref{eq:compensator} applied to the higher-order variation process with parameter $m$ yields
    \begin{equation}
        V_{m,t} = \sum_{j=1}^t \bbE\big[ |\Delta M_j|^m \,|\, \calF_{j-1}\big]
    \end{equation}
    with $\Delta M_t = M_t - M_{t-1}$, in agreement with the statement of \cref{simplevdg}.  Finally, since we assumed that $\bbE[M_t^2] < \infty$ for all $t$ in \cref{simplevdg}, the locally square integrable condition follows by choosing the trivial sequence of stopping times, $\tau_k = k$.

	
	\section{Proof of Lemma \ref{zconcentration} (Concentration of $\sum_j Z^j$)} \label{app:concentration}
	
	\cref{zconcentration} is restated as follows.
	
	\zconcentration*
	
	\begin{proof}
        Recall that $\bbE_{t-1}[\cdot]$ denotes expectation conditioned on the history up to index $t-1$. 
		Using the notation of \cref{simplevdg}, we let $M_t = \sum_{j\leq t} Z^j$, which yields $\Delta M_t = Z^t$.  The definition of $Z^t$ in \eqref{Zdefinition} ensures that $\bbE_{t-1}[Z^t] = 0$, so that $M_t$ is a martingale.  In addition, we have
        \begin{equation}
            V_{m,t} = \sum_{j=1}^t \bbE_{j-1}[|\Delta M_j|^m] = \sum_{j=1}^t \bbE_{j-1}[|Z^j|^m].
        \end{equation}

		To use~\cref{simplevdg}, we need to bound $\sum_{j=1}^t \bbE_{j-1}[|Z^j|^m]$ for some appropriate choices of $K$ and $R_t$ in \eqref{eq:Vmt}. The conditional moments of $|Z^j|$ are the central conditional moments of $Q^j$:
		\begin{align}
		\bbE_{j-1}[|Z^j|^m] &= \bbE_{j-1}[|Q^j - \bbE_{j-1}[Q^j]|^m] \label{subZ}\\
		&\leq \bbE_{j-1}[2^m(|Q^j|^m + |\bbE_{j-1}[Q^j]|^m )] \label{2ab} \\
		&\leq 2^{m+1} \bbE_{j-1}[|Q^j|^m] \label{eq1},
		\end{align}
		where \eqref{subZ} follows from the definition of $Z^j$ in \eqref{Zdefinition}, \eqref{2ab} uses $|a-b| \le 2\max\{|a|,|b|\}$, and \eqref{eq1} follows from Jensen's inequality ($|\bbE[Q^j]|^m \leq \bbE[|Q^j|^m]$). Furthermore, we have that
		\begin{align}
		\bbE_{j-1}[|Q^j|^m] &= \bbE_{j-1}[|(\lambda p^j \cdot \xtil^j - \ytil^j)(p^j - w/\lambda)\cdot \xtil^j|^m] \label{eq:apply_defQ}\\
		&\leq \bbE_{j-1}[|(\lambda p^j \cdot \xtil^j - \ytil^j)|^{2m}]^{1/2}  \bbE_{j-1}[|(p^j - w/\lambda)\cdot \xtil^j|^{2m}]^{1/2}, \label{eq:apply_CS}
		\end{align}
		where \eqref{eq:apply_defQ} uses the definition of $Q^j$ in \eqref{Qdefinition}, and \eqref{eq:apply_CS} follows from the Cauchy-Schwartz inequality.  Both of the averages in \eqref{eq:apply_CS} contain Gaussian random variables (with $p^j$ fixed due to the conditioning); we proceed by establishing an upper bound on the variances. Since $(\xtil^j, \ytil^j) = \frac{1}{B\sqrt{\nu_{\max}(\lambda +1)}} (x^j,y^j)$, the definition of $\nu_{\max}$ (see \eqref{eq:nu}) implies that each coordinate has a variance of at most $\big(\frac{1}{B\sqrt{\lambda+1}}\big)^2$. Then, using that $\sum_{i} p^j_i = 1$, we have 
		\begin{align}
		\mbox{Var}(\lambda p^j \cdot \xtil^j - \ytil^j) &\leq (\lambda + 1)^2 \max_{z \in \{\xtil^j_1, \dotsc, \xtil^j_n \ytil^j\}} \mbox{Var}(z) \\
		&\leq  \frac{\lambda + 1}{B^2}, \label{eq:var1}
		\end{align}
		and similarly, using $\sum_{i} p^j_i = 1$ and $\|w\| = \lambda$ (see Footnote \ref{foot:omitted_step}),
		\begin{align}
		\mbox{Var}((p^j - w/\lambda) \cdot \xtil^j) &\leq \frac{4}{(\lambda + 1)B^2}.  \label{eq:var2}
        \end{align}
		Next, we use the standard fact that if $N$ is a Gaussian random variable with mean $0$ and variance $\sigma$, then
		\begin{align}
		\bbE[N^p] = 
		\begin{cases}
		0 & \mbox{ if }p\mbox{ is odd}\\
		\sigma^p (p-1)!! &\mbox{ if }p\mbox{ is even.}
		\end{cases} \label{eq:gaussian_moments}
		\end{align}
		It then follows from \eqref{eq1} and \eqref{eq:var1}--\eqref{eq:gaussian_moments} that
		\begin{align}
		\bbE_{j-1}[|Z^j|^m] &\leq 2^{m+1} \bbE_{j-1}[|(\lambda p^j \cdot \xtil^j - \ytil^j)|^{2m}]^{1/2}  \bbE_{j-1}[|(p^j - w/\lambda)\cdot \xtil^j|^{2m}]^{1/2} \\
		&\leq 2^{m+1} \left(\left( \frac{\lambda + 1}{B^2} \right)^{2m} (2m-1)!! \left( \frac{4}{(\lambda + 1)B^2} \right)^{2m} (2m-1)!!\right)^{1/2}\\
		&= 2^{m+1} \frac{4^m}{B^{4m}} (2m-1)!! \\
		&= 2^{m+1} \frac{4^m}{B^{4m}} (1\cdot 3 \cdot \ldots \cdot (2m-1)) \\
		&\leq 2^{m+1} \frac{4^m}{B^{4m}} (2\cdot 4 \cdot \ldots \cdot 2m) \\
		&= 2 \cdot 4^{m} \frac{4^m}{B^{4m}} m! \label{eq:second_last} \\
		&= \frac{m!}{2} \left(\frac{16}{B^4}\right)^{m-2} \frac{2^{10}}{B^8},
		\end{align}
		and summing over $j=1,\dotsc,t$ gives
		\begin{align}
		\sum_{j=1}^t\bbE_{k-1}[|Z^j|^m] &\leq \frac{m!}{2} \left(\frac{16}{B^4}\right)^{m-2}  \frac{2^{10} t}{B^8}.
		\end{align}
		Hence, using the notation of ~\cref{simplevdg}, it suffices to set $K =\frac{16}{B^4}$ and $R_t = \frac{2^{10} t}{B^8}$. Plugging everything in, we get
		\begin{align}
		\bbP\left(\sum_{j=1}^T Z^j>a\right) &<  \exp \left( -\frac{a^2}{32a \frac{1}{B^4} + 2^{10} \frac{ T}{B^8} } \right).
		\end{align}
		Let $a = 2^{10} \sqrt{T \log \frac{1}{\delta}}$. Then, since $B = \sqrt{2 \log \frac{2pT}{\delta}}$ is always greater then $\sqrt{\log\frac{1}{\delta}}$, we obtain
		\begin{align}
		\bbP\left(\sum_{j=1}^t Z^j > 2^{10} \sqrt{T \log \frac{1}{\delta}} \right) &\le \frac{\delta}{2}.
		\end{align}
		By replacing $Z^j$ by $-Z^j$ above, we get a symmetric lower bound on $\sum_j Z^j$, as all the moments used above remain the same. Applying the union bound, we get that $|\sum_{j=1}^T Z^j| = O\big(\sqrt{T \log \frac{1}{\delta}}\big)$ with probability at least $1-\delta$.
	\end{proof}
	
	\section{Proof of Lemma \ref{bernstein} (Concentration of Empirical Risk)} \label{app:conc_risk}
	
	\cref{bernstein} is restated as follows. 
	
	\bernstein*

	\begin{proof}
		We first derive a simple equality:
		\begin{align}
		\bbE[(v\cdot a^j - b^j)^2] &= \bbE\big[ \bbE[(v\cdot a^j - b^j)^2 \,|\, a^j] \big] \\
		&= \bbE\big[ \big(\bbE[v\cdot a^j - b^j \,|\, a^j]\big)^2 + \mbox{Var}[b^j \,|\, a^j] \big] \label{eq:var_step}\\
		&= \bbE[(v\cdot a^j - w \cdot a^j)^2] + \bbE\big[\mbox{Var}[b^j \,|\,a^j]\big] \label{eq:lem_step}\\
		&= \varepsilon (v) + \Xi, \label{eq:def_step}
		\end{align}
        where \eqref{eq:var_step} uses $\mbox{Var}[Z] = \bbE[Z^2] - (\bbE[Z])^2$, \eqref{eq:lem_step} uses the second part of \cref{gaussian}, and \eqref{eq:def_step} uses the definitions of $\varepsilon (v)$ and $\Xi$.
		
		In the following, we recall Bernstein's inequality.
		
		\begin{lemma}
            {\em \citep[Corollary~2.11]{boucheron2013concentration}}
			Let $Z_1, \dots, Z_n$ be independent real-valued random variables, and assume that there exist positive numbers $\vartheta$ and $c$ such that
			\begin{align}
			\sum_{i=1}^n \bbE[(Z_i)_+^2] &\leq \vartheta \\
			\sum_{i=1}^n \bbE[(Z_i)_+^q] &\leq \frac{q!}{2} \vartheta \cdot c^{q-2},
			\end{align}
			where $(x)_+ = \max\{x,0\}$. Letting $S = \sum_{i=1}^n\big(Z_i - \bbE[Z_i]\big)$, we have for all $t > 0$ that
			\begin{align}
				\bbP(S \geq t) \leq \exp\left(- \frac{ t^2 }{2(\vartheta + ct)} \right).
			\end{align}
		\end{lemma}	
		
		We would like to use Bernstein's inequality to bound the deviation of
        \begin{equation}
            \frac{1}{M}\sum_{j=1}^M \Big( (v \cdot a^j - b^j)^2 - \Xi - \varepsilon (v) \Big)
        \end{equation}
        from its mean value $0$. To do so, we need to find constants $\vartheta$ and $c$ as described in the statement of Bernstein's inequality above.

        Recall that $\nu_{\max}$ upper bounds the variance of any marginal variable in each unnormalized sample, and that $(a^j,b^j)$ are samples normalized by $B\sqrt{\nu_{\max}(\lambda +1)}$ with $B = \sqrt{2 \log \frac{2pT}{\delta}} \ge 1$.  Hence, the entries of $(a^i,b^i)$ have variance at most $\frac{1}{\lambda + 1}$, and since $\|v\|_1 \le \lambda$, this implies that $v\cdot a^j - b^j$ has variance at most $\lambda+1$.

		Using the expression for the moments of a Gaussian distribution (see \eqref{eq:gaussian_moments}), it follows that
		\begin{align}
		\bbE [(v\cdot a^j - b^j)^4 ] &\le 8(\lambda + 1)^2, \\
		\bbE [(v\cdot a^j - b^j)^{2m}] &\le (2m-1)!! (\lambda + 1)^{m}  \\
		&\le 2^m m! (\lambda + 1)^{m} \label{eq:similar}\\
		&= \frac{m!}{2} (8(\lambda+1)^2 ) (2(\lambda+1) )^{m-2},
		\end{align}
        where \eqref{eq:similar} is established in the same way as \eqref{eq:second_last}.
		Since $(v\cdot a^j - b^j)^2$ is a non-negative random variable, the non-central moments bound the central moments from above. Hence, it suffices to let $\vartheta = 8(\lambda + 1)^2$ and $c = 2(\lambda+1)$, and we obtain from Bernstein's inequality that
		\begin{align}
		\bbP\left( \left|\sum_{j=1}^M \Big( \left(v \cdot a^j - b^j \right)^2 - \Xi - \varepsilon(v) \Big) \right| \geq \gamma M  \right) \leq  \exp\left(\frac{- \gamma^2 M^2}{2(8(\lambda + 1)^2 + 2(\lambda+1) \gamma M)}\right).
		\end{align}
		To simplify the notation, we let $M_0$ be such that $M = (\lambda+1) M_0$, which yields
		\begin{align}
		\bbP\left( \left|\frac{1}{M}\sum_{j=1}^M \Big( \left(v \cdot a^j - b^j \right)^2 - \Xi - \varepsilon(v) \Big) \right| \geq \gamma  \right) &\leq \exp\left(\frac{- \gamma^2 M_0^2}{16 + 2 \gamma M_0 }\right).
		\end{align}
		If $\gamma M_0 \geq 1$, then the right hand side is less than or equal to $\exp\big( \frac{-\gamma M_0}{18} \big)$. Otherwise, if $ \gamma M_0 < 1$, then the right hand side is less than $\exp\big( \frac{-\gamma^2 M_0^2}{18} \big)$. It follows that to have a deviation of $\gamma$ with probability at most $\rho$, it suffices to set $M_0 = \frac{18 \log(1/\rho)}{\gamma}$. 
		Recalling that $M = (\lambda+1) M_0$, it follows that with $M = 18 (\lambda+1) \frac{\log(1/\rho)}{\gamma}$, we attain the desired target probability $\rho$.
	\end{proof}

    \section{Proof of \cref{risktonorm} (Low Risk Implies an $\ell_{\infty}$ Bound)}    \label{pf_risktonorm}

	\cref{risktonorm} is restated as follows, and refers to the setup described in Section \ref{sec:proof}.

    \risktonorm*
        
    \begin{proof}    
        Recall that $\varepsilon(v) = \bbE[((v-w)\cdot X_{\bar{i}})^2]$, where $w = \big(\frac{-\theta_{ij}}{\theta_{ii}}\big)_{j\ne i}$ is the neighborhood weight vector of the node $i$ under consideration, and $X_{\bar{i}} = (X_j)_{j \ne i}$.  To motivate the proof, note from \cref{gaussian} that $X_i = \eta_i + \sum_{j\ne i} (-\theta_{ij}/\theta_{ii})X_j$, where $\eta_i$ is an $\calN\big(0,\frac{1}{\theta_{ii}}\big)$ random variable independent of $\{X_j\}_{j \ne i}$, from which it follows that $\Var(X_i) \geq \Var(\eta_i) = 1/\theta_{ii}$.  In the following, we apply similar ideas to $(v-w)\cdot X_{\bar{i}}$.

        Specifically, for an arbitrary index $i^* \ne i$, we can lower bound the expected risk $\varepsilon(v)$ as follows:
		\begin{align}
            &\bbE[((v-w)\cdot X_{\bar{i}})^2] \nonumber \\
    		&=\mathrm{Var}((v-w)\cdot X_{\bar{i}}) \label{eq:sq_to_var} \\
    		&= \mathrm{Var}\left( \sum_{j\ne i} (v_j - w_j) X_j \right) \\
    		&= \mathrm{Var}\left( (v_{i^*} - w_{i^*})X_{i^*} + \sum_{j \notin \{i, i^*\}} (v_j - w_j) X_j \right) \\
    		&= \mathrm{Var} \bigg((v_{i^*} - w_{i^*})\eta_{i^*} - (v_{i^*} - w_{i^*})\frac{\theta_{i^* i}}{\theta_{i^* i^*}}X_i   + \sum_{j \notin \{i, i^*\}} \left((v_j - w_j) - (v_{i^*} - w_{i^*})\frac{\theta_{i^*j}}{\theta_{i^*i^*}}\right) X_j \bigg) \label{eq:apply_i_lem}\\
    		&= \mathrm{Var}((v_{i^*} - w_{i^*})\eta_{i^*} ) + \mathrm{Var}\bigg( -(v_{i^*} - w_{i^*})\frac{\theta_{i^* i}}{\theta_{i^* i^*}}X_i + \sum_{j \notin \{i, i^*\}} \left((v_j - w_j) - (v_{i^*} - w_{i^*})\frac{\theta_{i^* j}}{\theta_{i^* i^*}}\right) X_j \bigg) \label{byindep} \\ 
    		&\geq \mathrm{Var}((v_{i^*} - w_{i^*})\eta_{i^*} ) \\
    		&= |v_{i^*} - w_{i^*}|^2 \mathrm{Var}(\eta_{i^*}),
		\end{align}
		where \eqref{eq:sq_to_var} follows since $\bbE[X_{\bar{i}}] = 0$, \eqref{eq:apply_i_lem} follows from the first part of \cref{gaussian} applied to node $i^*$, and \eqref{byindep} uses the independence of $\eta_{i^*}$ and $X_{\bar{i^*}}$.  Since $\mathrm{Var}(\eta_{i^*}) = \frac{1}{\theta_{i^*i^*}}$ and $\varepsilon(v) \le \epsilon$, this gives $|v_{i^*} - w_{i^*}| \leq \sqrt{\epsilon \theta_{i^*i^*}} \leq \sqrt{\epsilon \theta_{\mathrm{max}}}$.  Then, since this holds for all $i^* \ne i$, we deduce that $\|v-w\|_{\infty} \leq \sqrt{\epsilon \theta_{\mathrm{max}}}$, as desired.
	\end{proof}

\end{document}